\documentclass[11pt]{article}

\usepackage{fullpage}

\usepackage{microtype}
\usepackage{graphicx}
\usepackage{booktabs} 

\usepackage{hyperref}

\usepackage{algorithm}
\usepackage{algorithmic}

\usepackage{amsfonts}       
\usepackage{nicefrac}       

\usepackage{mathtools}
\usepackage{color}
\usepackage{multirow}

\newcommand{\inner}[1]{\left\langle#1\right\rangle}

\def\R{\mathbb{R}}

\def\bzero{{\mathbf 0}}
\def\bone{{\mathbf 1}}

\def\bh{{\mathbf h}}
\def\bl{{\mathbf l}}
\def\bm{{\mathbf m}}
\def\bp{{\mathbf p}}
\def\bq{{\mathbf q}}

\def\bv{\mathbf v}
\def\bw{{\mathbf w}}
\def\bx{{\mathbf x}}
\def\by{{\mathbf y}}

\def\bs{{\mathbf s}}
\def\bt{{\mathbf t}}

\def\bA{\mathbf A}
\def\bB{\mathbf B}
\def\bC{\mathbf C}
\def\bD{\mathbf D}

\def\bI{\mathbf I}

\def\bM{\mathbf M}

\def\calM{\mathcal M}
\def\calS{\mathcal S}
\def\calT{\mathcal T}
\def\calB{\mathcal B}

\def\bS{\mathbf S}
\def\bT{\mathbf T}
\def\bU{{\mathbf U}}
\def\bV{{\mathbf V}}
\def\bW{{\mathbf W}}

\def\bZ{{\mathbf Z}}

\def\argmax{\mathop{\rm arg\,max}\limits}
\def\argmin{\mathop{\rm arg\,min}\limits}

\def\maxop{\mathop{\rm max}\limits} 
\def\minop{\mathop{\rm min}\limits}

\newcommand{\trace}{{\rm trace}}

\newcommand\norm[1]{\left\lVert#1\right\rVert}

\usepackage{amsmath}
\usepackage{amsthm}

\newtheorem{Def}{Definition}[section]
\newtheorem{Thm}[Def]{Theorem}
\newtheorem{Lem}[Def]{Lemma}

\newcommand*\samethanks[1][\value{footnote}]{\footnotemark[#1]}

\title{\LARGE \bf Efficient Robust Optimal Transport\\ with Application to Multi-Label Classification}

\begin{document}

 \author{Pratik Jawanpuria\thanks{Microsoft India. Email: \{pratik.jawanpuria, bamdevm\}@microsoft.com.} \quad N T V Satyadev\thanks{Vayve Technologies. Email: tvsatyadev@gmail.com.} \quad Bamdev Mishra\samethanks[1]}


\maketitle
\thispagestyle{empty}
\pagestyle{empty}


\begin{abstract}
Optimal transport (OT) is a powerful geometric tool for comparing two distributions and has been employed in various machine learning applications. In this work, we propose a novel OT formulation that takes feature correlations into account while learning the transport plan between two distributions. We model the feature-feature relationship via a symmetric positive semi-definite Mahalanobis metric in the OT cost function. For a certain class of regularizers on the metric, we show that the optimization strategy can be considerably simplified by exploiting the problem structure. For high-dimensional data, we additionally propose suitable low-dimensional modeling of the Mahalanobis metric. Overall, we view the resulting optimization problem as a non-linear OT problem, which we solve using the Frank-Wolfe algorithm. Empirical results on the discriminative learning setting, such as tag prediction and multi-class classification, illustrate the good performance of our approach.

\end{abstract}

\section{Introduction}\label{sec:intro}
Optimal transport \cite{peyre19a} has become a popular tool in diverse machine learning applications such as domain adaptation \cite{courty17a,jawanpuria20a}, multi-task learning \cite{janati19a}, natural language processing \cite{melis18a,mjaw20b}, prototype selection \cite{gurumoorthy21a}, and computer vision \cite{rubner00}, to name a few. The optimal transport (OT) metric between two probability measures, also known as the Wasserstein distance or the earth mover's distance (EMD), defines a geometry over the space of probability measures \cite{villani09a} and evaluates the minimal amount of work required to transform one measure into another with respect to a \textit{given} ground cost function (also termed as the ground metric or simply the cost function). Due to its desirable metric properties, the Wasserstein distance has also been popularly employed as a loss function in discriminative and generative model training \cite{frogner15a,genevay18a,arjovsky17a}. 

The ground cost function may be viewed as a user-defined parameter in the OT optimization problem. Hence, the effectiveness of an OT distance depends upon the suitability of the (chosen) cost function. Recent works have explored learning the latter by maximizing the OT distance over a set of ground cost functions \cite{paty19a,paty20a,dhouib20a}. From a dual perspective, this may also be viewed as an instance of robust optimization \cite{bental09a,bertsimas11a} where the transport plan is learned with respect to an adversarially chosen ground cost function. Works such as \cite{paty19a,dhouib20a} parameterize the cost function with a Schatten norm regularized Mahalanobis metric to learn a OT distance robust to high dimensional (noisy) data. A bottleneck in such explorations is the high computational complexity \cite{paty19a,dhouib20a} of computing the robust OT distances, which inhibit their usage as a loss function in large-scale learning applications. 

In this work, we explore OT distances that take feature correlations into account within the given input space. Specifically, we parameterize the cost function with a sparse Mahalanobis metric to learn OT distances robust to  spurious feature correlations present in the data. It should be noted that Mahalanobis metric is a symmetric positive semi-definite matrix and an additional sparsity-inducing constraint/regularization on it may further complicate the optimization problem, potentially leading to a high computational complexity for learning it. In this regard, our main contributions are the following.
\begin{itemize}
\item We discuss three families of sparsity-inducing regularizers for the Mahalanobis metric parameterized cost function:  entry-wise $p$-norm for $p\in(1,2]$, KL-divergence based regularization, and the doubly-stochastic regularization. We show that for those regularizers enforcement of the symmetric positive semi-definite constraint is {\it not} needed as the problem structure {\it implicitly} learns a symmetric positive semi-definite matrix at optimality. The implications are two fold: a lower computational cost of computing the proposed robust OT distance and a simpler optimization methodology. 
\item To further reduce the computational burden of the proposed robust OT distance computation, we propose a novel $r$-dimensional modeling of the Mahalanobis metric for the studied robust OT problems, resulting in an even lower per-iteration computational cost, where $r \ll d$. The parameter $r$ provides an effective trade-off between computational efficiency and accuracy. 
\item We discuss how to use the robust distance as a loss in discriminative learning settings such as multi-class and multi-label classification problems. 
\item Empirically, we observe that computing the proposed robust OT distance is at least {\it $45$ times faster} than existing Mahalanobis metric parameterized robust OT distances \cite{paty19a,dhouib20a} and also obtains better generalization performance as a loss function in discriminative settings. 
\end{itemize}

The outline of the paper is as follows. Section~\ref{sec:related} presents a brief overview of the relevant OT literature. Section \ref{sec:formulations} discusses the proposed formulations. In Section~\ref{sec:optimization}, we discuss the optimization methodology. 
The setup of using the robust OT distance in learning problems is detailed in Section \ref{sec:ROTloss}. In Section \ref{sec:experiments}, we present the empirical results.  


\section{Background and related work}\label{sec:related}
Given two probability measures $\mu_1$ and $\mu_2$ over metric spaces $\calS$ and $\calT$, respectively, the optimal transport problem due to \cite{kantorovich42a} aims at finding a transport plan $\gamma$ as a solution to the following problem:
\begin{equation}\label{eqn:optimal-transport}
    {\rm W}_c(\mu_1,\mu_2) = \inf_{\gamma\in\Pi(\mu_1,\mu_2)} \int_{\calS\times\calT} c(\bs,\bt) d\gamma(\bs,\bt),
\end{equation}
where $\Pi(\mu_1,\mu_2)$ is the set of joint distributions with marginals $\mu_1$ and $\mu_2$ and $c:\calS \times \calT \rightarrow \R_{+}:(\bs,\bt)\rightarrow c(\bs,\bt)$ represents the transportation cost function. In several real-world applications, the distributions $\mu_1$ and $\mu_2$ are not available; instead samples from them are given. Let $\{\bs_i\}_{i=1}^m\in\R^{d}$ and $\{\bt_j\}_{j=1}^n\in\R^d$ represent $m$ and $n$ iid samples from $\mu_1$ and $\mu_2$, respectively. Then, empirical estimates of $\mu_1$ and $\mu_2$ supported on the given samples, defined as \begin{equation}
    \hat{\mu}_1\coloneqq\sum_{i=1}^m \bp_i\delta_{\bs_i},\;\;\;\;\;  \hat{\mu}_2\coloneqq\sum_{j=1}^n \bq_j\delta_{\bt_j},
\end{equation}
can be employed for computing the OT distance. Here,  $\delta$ is the delta function and $\bp\in\Delta_m$ and $\bq\in\Delta_n$ are the probability vectors, with $\Delta_d\coloneqq\{\bx\in\R^{d}|\bx\geq\bzero;\bx^\top\bone=1\}$. In this setting, the OT problem (\ref{eqn:optimal-transport}) may be rewritten as:
\begin{equation}\label{eqn:optimal-transport-discrete}
    {\rm W}_c(\bp,\bq) = \min_{\gamma\in\Pi(\bp,\bq)} \inner{\gamma,\bC},
\end{equation}
where $\bC$ is the $m\times n$ ground cost matrix with $\bC_{ij}=c(\bs_i,\bt_j)$ and $\Pi(\bp,\bq)=\{\gamma\in\R^{m\times n}|\gamma\geq\bzero;\gamma\bone=\bp;\gamma^\top\bone=\bq\}$. We obtain the popular $2$-Wasserstein distance by setting the cost function as the squared Euclidean function, i.e., $c(\bs,\bt)=\norm{\bs-\bt}^2$ when $\bs,\bt\in\R^d$. The $2$-Wasserstein distance can equivalently be reformulated as follows \cite{paty19a}:
\begin{equation}\label{eqn:W22distance}
    {{\rm W}}_2^2(\bp,\bq) = \minop_{\gamma\in\Pi(\bp,\bq)} \inner{\bV_\gamma,\bI},
\end{equation}
where $\bV_\gamma=\sum_{ij} (\bs_i-\bt_j)(\bs_i-\bt_j)^\top \gamma_{ij}$ is the weighted second-order moment of all source-target displacements. 

Recent works \cite{paty19a,kolouri19a,deshpande19a} have studied \textit{minimax} variants of the Wasserstein distance that aim at maximizing the OT distance in a projected low-dimensional space. From a duality perspective, these variants may also be viewed as instances of robust OT distance as they learn transport plan corresponding to worst possible ground cost function.  
Paty and Cuturi~\cite{paty19a} propose a robust variant of the ${\rm W}_2^2$ distance, termed as the Subspace Robust Wasserstein (SRW) distance, as follows: 
\begin{equation}\label{eqn:SRW1}
    {{\rm SRW}}_k^2(\bp,\bq) = \minop_{\gamma\in\Pi(\bp,\bq)}\maxop_{\bM\in\calM}\ \inner{\bV_\gamma,\bM},
\end{equation}
where the domain $\calM$ is defined as $\calM=\{\bM:\bzero\preceq\bM\preceq\bI {\rm \ and \ } \trace(\bM)=k\}$. 
It should be noted that $\inner{\bV_\gamma,\bM}=\int_{\calS\times\calT} c_\bM(\bs,\bt) d\gamma(\bs,\bt)$, where $c_{\bM}(\bs,\bt)=(\bs-\bt)^\top\bM(\bs-\bt)$ is a cost function parameterized by a Mahalanobis metric (symmetric positive semi-definite matrix) $\bM$ of size $d\times d$. 
Similarly, Dhouib~et~al.~\cite{dhouib20a} proposed a variant of robust OT distance (\ref{eqn:SRW1}), but with the domain $\calM$ defined as $\calM=\{\bM:\bM\succeq\bzero {\rm \ and\ } \norm{\bM}_{*p}=1\}$, where $\norm{\cdot}_{*p}$ denotes the Schatten $p$-norm regularizer, i.e., $\norm{\bM}_{*p}\coloneqq(\sum_i \sigma_i(\bM)^p)^\frac{1}{p}$. Here, $\sigma_i(\bM)$ denotes the $i$-th largest eigenvalue of $\bM$.

Both \cite{paty19a,dhouib20a} pose their Mahalanobis metric parameterized robust OT problems as optimization problems over the metric $\bM$. This usually involves satisfying/enforcing the positive semi-definite (PSD) constraint at every iteration, which requires eigendecomposition operations having $O(d^3)$ computational cost. In this context, Dhouib~et~al.~\cite{dhouib20a} have remarked in their work that ``\textit{...PSD constraints increase considerably the computational burden of any optimization problem, yet they are necessary for the obtained cost function to be a true metric}''. 
In their formulation, \cite{dhouib20a} needs to explicitly enforce the PSD constraint (e.g., via eigendecomposition) for all Schatten $p$-norm regularizers except for $p=2$.

\section{Novel formulations for robust OT}\label{sec:formulations}
We consider a general formulation of Mahalanobis metric parameterized robust optimal transport problem as follows:
\begin{equation}\label{eqn:MahalanobisRobustOT}
    {\rm W}_{\rm ROT}(\bp,\bq)\coloneqq\minop_{\gamma\in\Pi(\bp,\bq)}\ f(\gamma),
\end{equation}
where the function $f:\R^{m\times n}\rightarrow\R : \gamma \mapsto f(\gamma)$ is defined as 
\begin{equation}\label{eqn:fgamma}
    f(\gamma)\coloneqq\maxop_{\bM\in\calM}\ \inner{\bV_\gamma,\bM}.
\end{equation}
Here, $\calM=\{\bM:\bM\succeq\bzero {\rm \ and \ }\Omega(\bM) \leq 1\}$ and $\Omega(\cdot)$ is a convex regularizer on the set of positive semi-definite matrices. 
It should be noted that (\ref{eqn:MahalanobisRobustOT}) is a convex optimization problem. Moreover, by the application of the Sion-Kakutani min-max theorem \cite{sion58}, Problem (\ref{eqn:MahalanobisRobustOT}) can be shown to be equivalent to its dual max-min problem:  $\max_{\bM\in\calM}\min_{\gamma\in\Pi(\bp,\bq)}\inner{\bV_\gamma,\bM}$. 

As discussed in Section~\ref{sec:related}, robust OT distances studied in \cite{paty19a,dhouib20a} can be obtained from (\ref{eqn:MahalanobisRobustOT}) by considering appropriate Schatten-norm based regularizers as $\Omega(\cdot)$. It is well known that Schatten-norm regularizers influence sparsity of the eigenvalues of $\bM$. In contrast, we study novel robust OT formulations based on sparsity promoting regularizers on the entries of $\bM$. A sparse Mahalanobis metric is useful in avoiding spurious feature correlations~\cite{rosales06a,qi09a}.

\subsection{Element-wise $p$-norm regularization on $\bM$}\label{sec:p_norm}

We begin by discussing the element-wise $p$-norm regularization on the Mahalanobis metric $\bM$: 
$\calM=\{\bM:\bM\succeq\bzero {\rm \ and \ } \norm{\bM}_{p}\leq1\}$ in (\ref{eqn:fgamma}), where $\norm{\bM}_{p}=(\sum_{ij} |\bM_{ij}|^p)^{\frac{1}{p}}$. For $p$ in between $1$ and $2$, the entry-wise $p$-norm regularization induces a sparse structure on the metric $\bM$. This family of element-wise $p$-norm  regularizers includes the popular Frobenius norm at $p=2$. The following result provides an efficient reformulation of the robust OT problem (\ref{eqn:MahalanobisRobustOT}) with the above defined $\calM$ for a subset of the element-wise $p$-norm regularizers on $\bM$.


\begin{Thm}\label{thm:p-norm}
Consider Problem (\ref{eqn:MahalanobisRobustOT}) with $\calM=\{\bM:\bM\succeq\bzero {\rm \ and \ } \norm{\bM}_p\leq1\}$, $p=\frac{2k}{2k-1}$, and $k\in\mathbb{N}$. Then, (\ref{eqn:MahalanobisRobustOT}) is equivalent to the following optimization problem:
\begin{equation}\label{eqn:proposed-p-norm-min}
{{\rm W}}_{{\rm P}}(\bp,\bq)\coloneqq\minop_{\gamma\in\Pi(\bp,\bq)}\ \norm{\bV_\gamma}_{2k},
\end{equation}
where $\bV_\gamma=\sum_{ij} (\bs_i-\bt_j)(\bs_i-\bt_j)^\top \gamma_{ij}$. For a given $\gamma\in\Pi(\bp,\bq)$, the optimal solution of (\ref{eqn:fgamma}) is $\bM^{*} (\gamma) =\norm{\bV_{\gamma}}_{2k}^{1-2k}(\bV_{\gamma})^{\circ(2k-1)}$, where $\bA^{\circ (k)}$ denotes the $k$-th Hadamard power of a matrix $\bA$.
\end{Thm}
\begin{proof}
The proof is in Appendix \ref{supp:sec:proof_pnorm}. 
\end{proof}



From Theorem~\ref{thm:p-norm}, it should be observed that for a given $\gamma$, the optimal $\bM^{*}(\gamma)$ is an element-wise function of the matrix $\bV_\gamma$. An implication is that the computation of $\bM^*(\gamma)$ costs $O(d^2)$.

\subsection{KL-divergence regularization on $\bM$}\label{sec:KL}

We next consider the generalized KL-divergence regularization on the metric $\bM$, i.e., $\calM=\{\bM:\bM\succeq\bzero;{{\rm D}_{\rm KL}}(\bM,\bM_0)\leq 1\}$, where ${{\rm D}_{\rm KL}}(\bM,\bM_0)=\sum_{a,b} \bM_{ab}(\ln{\left(\bM_{ab}/{\bM_0}_{ab}\right)})-(\bM_{ab} - {\bM_{0}}_{ab})$ denotes the Bregman distance, with negative entropy as the distance-generating function, between the matrices $\bM$ and $\bM_0\succeq\bzero$. Here, $\bM_0$ is a given symmetric PSD matrix, which may be useful in introducing prior domain knowledge (e.g., a block diagonal matrix $\bM_0$ ensures grouping of features) and $\ln(\cdot)$ is the natural logarithm operation. 
The function $f(\gamma)$ in (\ref{eqn:fgamma}) with the above defined $\calM$ may be expressed equivalently in the Tikhonov form as
\begin{equation}\label{eqn:proposed-minimax-kl-divergence}
f(\gamma) = \maxop_{\bM\succeq \bzero}\  \inner{\bV_\gamma,\bM}-\lambda_{\bM} {{\rm D}_{\rm KL}}(\bM,\bM_0),
\end{equation}
where $\lambda_{\bM}>0$ is a regularization parameter. The following result provides an efficient reformulation of the KL-divergence regularized robust OT problem. 
\begin{Thm}\label{thm:kl-regularization}
The following problem is an equivalent dual of the convex problem (\ref{eqn:MahalanobisRobustOT}) where $f(\gamma)$ is as defined in (\ref{eqn:proposed-minimax-kl-divergence}):
\begin{equation}\label{eqn:proposed-min-kl-divergence}
\begin{array}{lll}
{{\rm W}}_{{\rm KL}}(\bp,\bq) \\
= \minop_{\gamma\in\Pi(\bp,\bq)} \lambda_{\bM} \bone^\top( \bM_0\odot e^{\circ (\bV_\gamma/\lambda_{\bM})} - \bM_0)\bone,
\end{array}
\end{equation}
where $e^{\circ\bA}$ denotes the Hadamard exponential (element-wise exponential) of a matrix $\bA$ and $\odot$ denotes the Hadamard product (element-wise product). For a given $\gamma\in\Pi(\bp,\bq)$, the optimal solution of  (\ref{eqn:proposed-minimax-kl-divergence}) is  $\bM^{*}(\gamma)=\bM_0\odot e^{\circ (\bV_{\gamma}/\lambda_{\bM})}$. 
\end{Thm}
\begin{proof}
We first characterize the relaxed unconstrained (i.e., without the PSD constraint) solution of (\ref{eqn:proposed-minimax-kl-divergence}), which is $\bM^{*}(\gamma)=\bM_0\odot e^{\circ (\bV_{\gamma}/\lambda_{\bM})}$. As element-wise exponential operation on symmetric PSD matrices preserves positive semi-definiteness, i.e., $\bM^{*}(\gamma)=\bM_0\odot e^{\circ (\bV_{\gamma}/\lambda_{\bM})}$ is also PSD, implying that this is also the optimal solution of (\ref{eqn:proposed-minimax-kl-divergence}). Putting $\bM^{*}(\gamma)$ in the objective function of (\ref{eqn:proposed-minimax-kl-divergence}) leads to (\ref{eqn:proposed-min-kl-divergence}).
\end{proof}

The proposed KL-regularization based OT formulation (\ref{eqn:proposed-minimax-kl-divergence}) is closely related to selecting discriminative features that maximize the optimal transport distance between two distributions. We formalize this in our next result. 
\begin{Lem}\label{lemma:feature-selection}
Let $\gamma^{*}$ be the solution of (\ref{eqn:proposed-min-kl-divergence}) with $\bM_0=\bI$. Then, $\gamma^{*}$ is also the solution of the following robust OT formulation:
\begin{equation}\label{eqn:feature-selection}
\minop_{\gamma\in\Pi(\bp,\bq)}\ \maxop_{\bm\in\Delta_d} \inner{\gamma,\bC_{\bm}}-\lambda_{\bm} {{\rm D}_{\rm KL}}(\bm,\bone),
\end{equation}
$\Delta_d$ is the $d$-dimensional simplex and $\bC_{\bm}$ is the ground cost matrix corresponding to the function $c_{\bm}(\bs_i,\bt_j)=\sum_{a=1}^d \bm_a (\bs_{ia}-\bt_{ja})^2$. Here, $\bs_{ia}$ and $\bt_{ja}$ denote the $a$-th coordinate of the data points $\bs_i$ and $\bt_j$, respectively. 
\end{Lem}

\begin{proof}
The proof is in Appendix \ref{supp:sec:proof_simplex}.
\end{proof}

The simplex constraint over feature weights in (\ref{eqn:feature-selection}) results in selecting features that maximize the OT distance. 
Thus, the proposed robust OT formulation (\ref{eqn:proposed-min-kl-divergence}) may equivalently be viewed as a robust OT formulation involving feature selection when $\bM_0=\bI$. Feature selection in the OT setting has also been explored in a concurrent work by Petrovich et al. \cite{petrovich20a}.

\subsection{Doubly-stochastic regularization on $\bM$}\label{sec:doubly_stochastic}

We further study the KL-regularization based robust OT objective (\ref{eqn:proposed-minimax-kl-divergence}) with a doubly-stochastic constraint on the Mahalanobis metric $\bM$, i.e.,
\begin{equation}\label{eqn:proposed-doubly-stochastic}
    {\rm W}_{\rm DS}(\bp,\bq) = \minop_{\gamma\in\Pi(\bp,\bq)} f(\gamma), 
\end{equation}
where
\begin{equation}\label{eqn:proposed-minimax-doubly-stochastic}
f(\gamma) = \maxop_{\bM\in\calM}\ \inner{\bV_\gamma,\bM} - \lambda_{\bM}{{\rm D}_{\rm KL}}(\bM,\bM_0),
\end{equation}
where $\calM=\{\bM:\bM\succeq\bzero ,\bM \geq\bzero,\bM\bone=\bone,\bM^\top\bone=\bone\}$. 
Learning a doubly-stochastic Mahalanobis metric is of interest in  applications such as graph clustering and community detection \cite{zass06a,arora11a,wang16a,douik18a,douik2019manifold}. Since $\bM$ is a symmetric PSD matrix, $\bM\bone=\bone \Leftrightarrow \bM^\top\bone=\bone$. 

In general, optimization over the set of symmetric PSD doubly-stochastic matrices is non trivial and computationally challenging. By exploiting the problem structure, however, we show that (\ref{eqn:proposed-minimax-doubly-stochastic}) can be solved efficiently. 
Our next result characterizes {the optimal} solution of (\ref{eqn:proposed-minimax-doubly-stochastic}). 
\begin{Thm}\label{thm:doubly-stochastic-regularization}
For a given $\gamma\in\Pi(\bp,\bq)$, {the optimal} solution of (\ref{eqn:proposed-minimax-doubly-stochastic}) has the form $\bM^{*}(\gamma)=\bD\left( \bM_0 \odot e^{\circ(\bV_\gamma/\lambda_{\bM})}\right)\bD$, where $\bD$ is a diagonal matrix with positive entries. The matrix $\bD$ can be computed using the Sinkhorn algorithm \cite{cuturi13a}. 
\end{Thm}
\begin{proof}
The proof is in Appendix \ref{supp:sec:proof_DS}.
\end{proof}

\section{Optimization}\label{sec:optimization}
In this section, we discuss how to solve (\ref{eqn:MahalanobisRobustOT}). But first, we discuss a novel low-dimensional modeling approach that helps reduce the computational burden and is amenable to our proposed regularizations in Section \ref{sec:formulations}.

\subsection{Low-dimensional modeling of $\bM$ with feature grouping} \label{sec:low_dimensional_metric_decomposition}

Section \ref{sec:formulations} discusses several regularizations on the Mahalanobis metric $\bM$ that lead to efficient computation of $\bM^{*}(\gamma)$, i.e., the solution to  (\ref{eqn:fgamma}), requiring $O(d^2)$ computations. This, though linear in the size of $\bM$, may be prohibitive for  high-dimensional data. Here, we discuss a particular low-dimensional modeling technique that addresses this computational issue. We consider our Mahalanobis metric $\bM$ to be of the {\it general} form 
\begin{equation}\label{eqn:factorizeMahalanobis}
    \bM=\bB\otimes\bI_{d_1},
\end{equation}
where $\otimes$ denotes the Kronecker product, $\bI_{d_1}$ is the identity matrix of size $d_1 \times d_1$, and  $\bB\succeq\bzero$ is a $r\times r$ symmetric positive semi-definite matrix with $d=d_1r$, where $r \ll d$. 
The modeling (\ref{eqn:factorizeMahalanobis}) reformulates the objective in (\ref{eqn:fgamma}) as
\begin{equation}\label{eqn:reshaping}
\begin{array}{l}
    \inner{\bV_\gamma,\bB\otimes\bI_{d_1}} = \inner{\sum_{ij}\gamma_{ij}(\bs_i-\bt_j)(\bs_i-\bt_j)^\top,\bB\otimes\bI_{d_1}}\\
    \qquad\qquad\qquad= \inner{\sum_{ij}\gamma_{ij}(\bS_i-\bT_j)^\top(\bS_i-\bT_j),\bB}\\
    \qquad\qquad\qquad= \inner{\bU_\gamma, \bB},
\end{array}
\end{equation}
where $\bS_i$ and $\bT_j$ are $d_1\times r$ matrices obtained by reshaping the vectors $\bs_i$ and $\bt_j$, respectively.

We observe that the proposed modeling (\ref{eqn:factorizeMahalanobis}) divides $d$ features into $r$ groups, each with $d_1$ features. Based on (\ref{eqn:reshaping}), the symmetric positive semi-definite matrix $\bB$ may be viewed as a Mahalanobis metric over the {\it feature groups}. In addition, it can be shown that any proposed regularization on the metric $\bM$ (in Section~\ref{sec:formulations}) transforms into an equivalent regularization on the ``group'' metric $\bB$. Equivalently, our robust OT problem is with $\gamma$ and 
$\bB$ (and not $\bM$). In case there is no feature grouping ($d_1 = 1$), learning of $\bB$ is same as $\bM$ and there is no computational benefit.

The function $f(\gamma)$, in the robust optimal transport problem (\ref{eqn:MahalanobisRobustOT}), is equivalently re-written as
\begin{equation}\label{eqn:factorizeMahalanobisRobustOT}
    f(\gamma) = \max_{\bB\in\calB}\ \inner{\bU_\gamma,\bB},
\end{equation}
where $\calB=\{\bB:\bB\succeq\bzero {\rm \ and \ } \Omega(\bB) \leq 1\}$. For the regularizers discussed in Section \ref{sec:formulations}, the computation of the solution $\bB^*(\gamma)$ of (\ref{eqn:factorizeMahalanobisRobustOT}) costs $O(r^2)$.

\subsection{Frank-Wolfe algorithm for (\ref{eqn:MahalanobisRobustOT})}\label{sec:algorithms}


A popular way to solve a convex constrained optimization problem (\ref{eqn:MahalanobisRobustOT}) is with the Frank-Wolfe (FW) algorithm, which is also known as the conditional gradient algorithm \cite{jaggi2013revisiting}. It requires solving a constrained linear minimization sub-problem (LMO) at every iteration. The LMO step boils down to solving the standard optimal transport problem (\ref{eqn:optimal-transport}). When regularized with an entropy regularization term, the LMO step admits a computationally efficient solution using the popular Sinkhorn algorithm \cite{cuturi13a}. The FW algorithm  for (\ref{eqn:MahalanobisRobustOT}) is shown in Algorithm \ref{alg:FW}, which only involves the gradient of the function.

\begin{algorithm}[tb]
   \caption{Proposed FW algorithm for (\ref{eqn:MahalanobisRobustOT})}
   \label{alg:FW}
   {
  {\bfseries Input:} Source distribution's samples $\{\bs_i\}_{i=1}^m\in\R^{d}$ and target distribution's samples $\{\bt_j\}_{j=1}^n\in\R^{d}$. Initialize $\gamma_0 \in \Pi(\mu_1, \mu_2)$. 
   
  {\bfseries for} $t = 0\ldots T$ {\bfseries do} 
  
  \qquad Compute $\nabla f_{\gamma}(\gamma_t)$ using Lemma \ref{lemma:gradient_and_directional_derivative}.

  \qquad LMO step: Compute $\hat{\gamma}_t\coloneqq \argmin_{\beta \in \Pi(\bp,\bq)} \langle \beta, \nabla f_{\gamma}(\gamma_t)\rangle$.

  \qquad Update $\gamma_{t+1} = (1-\theta)\gamma_t + \theta \hat{\gamma}_t$ for $\theta = \frac{2}{t+2}$.
 
  {\bfseries end for}

{\bfseries Output:} $\gamma^*$ and $\bM^* = \bM^*(\gamma^*)$.
    }
\end{algorithm}

We now show the computation of gradient $\nabla_{\gamma} f$ of (\ref{eqn:MahalanobisRobustOT}). We begin by noting that $\bV_\gamma = \bZ {\rm Diag}({\rm vec}(\gamma)) \bZ^\top$, where $\bZ$ is a $d\times mn$ matrix with $(i,j)$-th column as $(\bs_i-\bt_j)$, ${\rm Diag}(\cdot)$ acts on a vector and outputs the corresponding diagonal matrix, and ${\rm vec}(\cdot)$ vectorizes a matrix in the column-major order. 

\begin{Lem}\label{lemma:gradient_and_directional_derivative}
Let $\bM^*(\gamma_t)$ be the solution of the problem $\maxop_{\bM \in \mathcal{M}} \langle \bV_{\gamma_t},  \bM\rangle$ for a given $\gamma_t \in \Pi(\bp,\bq)$. Then, the gradient $\nabla_{\gamma} f (\gamma_t)$ of $f(\gamma)$ in (\ref{eqn:MahalanobisRobustOT}) with respect to $\gamma$ at $\gamma_t$ is
\begin{equation*}
    \nabla_{\gamma} f(\gamma_t) = {\rm vec}^{-1}({\rm diag}(\bZ^\top \bM^*(\gamma_t)\bZ)), 
\end{equation*}
where ${\rm diag}(\cdot)$ extracts the diagonal (vector) of a square matrix and ${\rm vec}^{-1}$ reshapes a vector into a matrix. 

\end{Lem}
\begin{proof}
The proof follows directly from the Danskin's theorem \cite{bertsekas1995nonlinear} and exploits the structure of $\bV_{\gamma}$.
\end{proof}

For the formulations presented in Section \ref{sec:formulations}, the computation of $\nabla_{\gamma} f$ costs $O(mnd^2)$. 
For the cost function (\ref{eqn:factorizeMahalanobisRobustOT}) obtained with the $r$-dimensional modeling of $\bM$, Lemma \ref{lemma:gradient_and_directional_derivative} can be appropriately modified. In this case, the cost of computation of $\nabla_{\gamma} f $ is $O(mnr^2)$.

\section{Learning with robust optimal transport loss}\label{sec:ROTloss}

In this section, we discuss the discriminative learning setup \cite{frogner15a} and the suitability of the proposed robust OT distances as a loss function in the learning setup. 

\subsection{Problem setup}
Consider the standard multi-label (or multi-class) problem over $L$ labels (classes) and given supervised training instances $\{\bx_j,\by_j\}_{j=1}^N$. Here, $\bx_j\in\R^M$ and $\by_j\in\{0,1\}^L$. The prediction function of $p$-th label is given by the {\it softmax} function
\begin{equation}\label{eqn:softmax}
    h_p(\bx;\bW)=\frac{e^{\inner{\bw_p,\bx}}}{\sum_{l=1}^L e^{\inner{\bw_l,\bx}}},
\end{equation}
where $\bW=[\bw_1,\ldots,\bw_L]$ is the model parameter of the multi-label problem. The prediction function $h(\bx;\bW)$ can be learned via the empirical risk minimization framework. 

Frogner et al.~\cite{frogner15a} propose employing the OT distance (\ref{eqn:optimal-transport-discrete}) as the loss function for multi-label classification problem as follows. For an input $\bx_j$, the prediction function $h(\bx_j;\bW)$ in (\ref{eqn:softmax}) may be viewed as a discrete probability distribution. Similarly, a $0-1$ binary ground-truth label vector $\by_j$ may be transformed into a discrete probability distribution by appropriate normalization: $\hat{\by}_j=\by_j/\by_j^\top \bone$. Given a suitable ground cost metric between the labels, the Wasserstein distance is employed to measure the distance between the prediction $h(\bx_j;\bW)$ and the ground-truth $\hat{\by}_j$. If the labels correspond to real-word entities, then a possible ground cost metric may be obtained from the word embedding vectors corresponding to the labels \cite{mikolov13a,bojanowski16a}. 

\begin{figure}[t]
\center
\includegraphics[width=.7\textwidth]{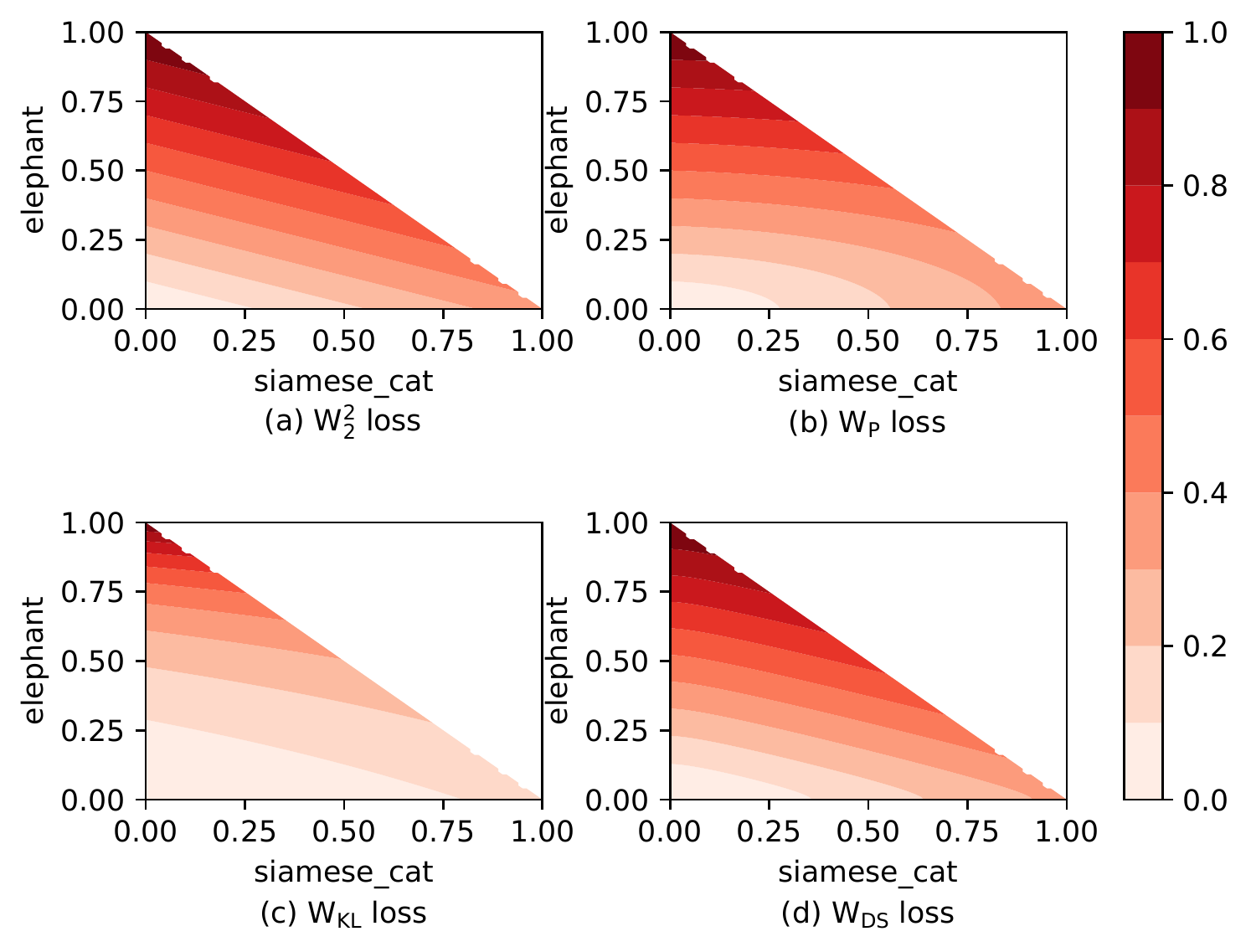}
\caption{Contour plots of various OT distance based loss functions. The labels are \{`Siamese$\_$cat', `Elephant', `Persian$\_$cat'\}, with `Persian$\_$cat' as the true label. The plots show the relative loss incurred by various possible predictions of the form $\bh=[x,y,1-x-y]^\top$, where $0\leq x,y\leq1$ and $x+y\leq 1$. We observe linear contours for the ${\rm W}_2^2$ loss function while non-linear contours for the proposed robust OT loss functions: ${\rm W}_{\rm P}$ with $k=1$, ${\rm W}_{\rm KL}$, and ${\rm W}_{\rm DS}$. The robust OT loss functions relatively penalize `Siamese$\_$cat' lesser than `Elephant', i.e., more lighter color in the feasible $x>y$ regions in plots (b), (c), and (d) than in plot (a). }\label{fig:cce}
\end{figure}



\subsection{Multi-label learning with the ${\rm W}_{\rm ROT}$ loss}
We propose to employ the robust OT distance-based loss in multi-label/multi-class problems. To this end, we solve the empirical risk minimization problem, i.e.,
\begin{equation}\label{eqn:multi-label}
    \minop_{\bW\in\R^{M\times L}}\ \frac{1}{N} \sum_{j=1}^N{\rm W}_{\rm ROT}(h(\bx_j;\bW),\hat{\by}_j),
\end{equation}
where ${\rm W}_{\rm ROT}$ is the robust OT distance-based function  (\ref{eqn:MahalanobisRobustOT}). Here, ${\rm W}_{\rm ROT}$ may be set to any of the discussed robust OT distance functions such as ${\rm W}_{\rm P}$ (\ref{eqn:proposed-p-norm-min}), ${\rm W}_{\rm KL}$ (\ref{eqn:proposed-min-kl-divergence}), and ${\rm W}_{\rm DS}$ (\ref{eqn:proposed-doubly-stochastic}). As discussed, \cite{frogner15a} employs ${\rm W}_p^p(h(\bx;\bW),\hat{\by})$ as the loss function in (\ref{eqn:multi-label}).

We analyze the nature of the ${\rm W}_2^2$-based loss function and the proposed ${\rm W}_{\rm ROT}$-based loss functions by viewing their contours plots for a three-class setting with labels as \{\textit{A},\textit{B},\textit{C}\}. We consider label \textit{C} as true label, i.e., $\hat{\by}=[0,0,1]^\top$, where the first dimension corresponds to label \textit{A} and consider predictions of the form $\bh=[a,b,1-a-b]$. Since $h$ is obtained from the softmax function (\ref{eqn:softmax}), we have $(a,b) \in \{(x,y):0\leq x,y\leq 1 {\rm \ and \ }x+y\leq 1\}$. The ground cost function is computed using the fastText word embeddings corresponding to the labels \cite{bojanowski16a}. We plot the contour maps of ${\rm W}_2^2(\bh,\hat{\by})$ and the proposed ${\rm W}_{\rm P}(\bh,\hat{\by})$, ${\rm W}_{\rm KL}(\bh,\hat{\by})$, and ${\rm W}_{\rm DS}(\bh,\hat{\by})$, as $(a,b)$ varies along the two-dimensional $X-Y$ plane. All the plots are made to same scale by normalizing the highest value of the loss to $1$.

In Figure~\ref{fig:cce}, we consider the labels as \{\textit{A},\textit{B},\textit{C}\}=\{`Siamese$\_$cat', `Elephant', `Persian$\_$cat'\}. The first and the third classes in this setting are similar, and the OT distance based loss functions should exploit this relationship via the ground cost function. With `Persian$\_$cat' as the true class, we observe that all the four losses in Figure~\ref{fig:cce} penalize `Elephant' more than `Siamese$\_$cat'. However, the contours of ${\rm W}_2^2$ loss in  Figure~\ref{fig:cce}(a) are linear, while those of the proposed ${\rm W}_{\rm P}$ with $k=1$ Figure~\ref{fig:cce}(b) are elliptical. On the other hand, the proposed ${\rm W}_{\rm KL}$ and ${\rm W}_{\rm DS}$ exhibit varying degree of non-linear contours. 
Overall, the proposed robust OT distance based loss functions enjoy more flexibility in the degree of penalization than the ${\rm W}_2^2$ loss, which may be helpful in certain learning settings. 


\subsection{Optimizing ${\rm W}_{\rm ROT}$ loss}

We solve Problem (\ref{eqn:multi-label}) using the standard stochastic gradient descent (SGD) algorithm. In each iteration of the SGD algorithm, we pick a training instance $j = \{1, \ldots, N \}$ and update the parameter $\bW$ along the negative of the gradient of the loss term ${\rm W}_{\rm ROT}(h(\bx_j;\bW),\hat{\by}_j))$ with respect to the model parameter $\bW$. 

We obtain it using the chain rule by computing $\nabla_\bW h(\bx_j;\bW)$ and $\nabla_{h(\bx;\bW)} {\rm W}_{\rm ROT}(h(\bx_j;\bW),\hat{\by}_j)$. While the expression for $\nabla_\bW h(\bx_j;\bW)$ is well studied, computing $\nabla_{h(\bx;\bW)} {\rm W}_{\rm ROT}(h(\bx_j;\bW),\hat{\by}_j)$ is non trivial as the loss ${\rm W}_{\rm ROT}(h(\bx_j;\bW),\hat{\by}_j)$ involves a min-max optimization problem  (\ref{eqn:MahalanobisRobustOT}). To this end, we consider a regularized version of ${\rm W}_{\rm ROT}$ by adding a negative entropy regularization term to (\ref{eqn:fgamma}). Equivalently, we consider the formulation for computing the robust OT distance between $h(\bx_j;
\bW)$ and $\hat{\by}_j$ as
\begin{equation}\label{eqn:rotminmax}
\begin{array}{ll}
  {\rm W}_{\rm ROT}(h(\bx_j;\bW),\hat{\by}_j) = \\
  \minop_{\gamma \in \Pi(h(\bx_j;\bW),\hat{\by}_j)} \maxop_{\bM \in \calM} \inner{\bV_\gamma,\bM} \\
  \qquad \qquad \qquad \qquad \qquad + \lambda_{\gamma} \sum_{pq} \gamma_{pq} {\rm ln}(\gamma_{pq}),
  \end{array}
\end{equation}
where $\lambda_{\gamma} >0$ and $\bV_{\gamma} = \sum_{pq} (\bl_p - \bl_q)(\bl_p - \bl_q)^\top \gamma_{pq}$ is of size $d \times d$. Here, $\bl_p$ is the ground embedding of $p$-th label of dimension $d$. The following lemma provides the expression for the gradient of ${\rm W}_{\rm ROT}(h(\bx_j;\bW),\hat{\by}_j)$ with respect to $h(\bx; \bW)$.

\begin{Lem}\label{lemma:gradient_rot}
Let $(\gamma^{*}, \bM^*(\gamma^*))$ denote the optimal solution of the robust OT problem (\ref{eqn:rotminmax}). Then, 
\begin{equation}\label{eqn:gradient_rot}
\begin{array}{lll}
\nabla_{h(\bx;\bW)} {\rm W}_{\rm ROT}(h(\bx_j;\bW),\hat{\by}_j) =  \frac{1}{L}\bA \bone - \frac{\bone ^\top \bA \bone}{L^2} \bone,
\end{array}
\end{equation}
where $\bone $ is the column vector of ones of size $L$, $\bA = \bC^*+ \lambda_{\gamma } ({\rm ln}(\gamma ^*) + \bone \bone^\top)$, and $\bC^*(\bl_p,\bl_q) = (\bl_p - \bl_q)^\top \bM^*(\gamma^*) (\bl_p - \bl_q)$ is of size $L \times L$.
\end{Lem}
\begin{proof}
The proof is in Appendix \ref{supp:sec:proof_gradient_rot}.
\end{proof}

For the multi-label setting, computation of the gradient $\nabla_{h(\bx;\bW)} {\rm W}_{\rm ROT}(h(\bx_j;\bW),\hat{\by}_j)$, shown in (\ref{eqn:gradient_rot}), in Lemma \ref{lemma:gradient_rot} costs $O(l_j Ld^2)$, where $l_j$ is the number of ground-truth labels for the $j$-th training instance. Using the modeling (\ref{eqn:factorizeMahalanobis}), the cost reduces to $O(l_jLr^2)$. In many cases, $l_j$ is much smaller than $L$. In the multi-class setting, the gradient computation cost reduces to costs $O(Ld^2)$ by setting $l_j = 1$. 
Using the modeling (\ref{eqn:factorizeMahalanobis}), it can be reduced to $O(Lr^2)$. 

Overall, in both multi-class/multi-label settings, the cost of the gradient computation in (\ref{eqn:gradient_rot}) scales linearly with the number of labels $L$ and quadratically with $r$. When $r \ll d$, optimization of the ${\rm W}_{\rm ROT}$ loss becomes computationally feasible for large-scale multi-class/multi-label instances.

\section{Experiments}\label{sec:experiments}
We evaluate the proposed robust optimal transport formulations in the supervised multi-class/multi-label setting discussed in Section~\ref{sec:ROTloss}. Our code is available at \url{https://github.com/satyadevntv/ROT4C}.


\subsection{Datasets and evaluation setup}
We experiment on the following three multi-class/multi-label datasets. 

\textbf{Animals} \cite{lampert13a}: This dataset contains $30\,475$ images of $50$ different animals. DeCAF features ($4096$ dimensions) of each image are available at \url{https://github.com/jindongwang/transferlearning/blob/master/data/dataset.md}. We randomly sample $10$ samples per class for training and the rest are used for evaluation.

\textbf{MNIST}: The MNIST handwritten digit dataset consists of images of digits $\{0,\ldots,9\}$. The images are of $28\times 28$ pixels, leading to $784$ features. The pixel values are normalized by dividing each dimension with $255$. We randomly sample $100$ images per class (digit) for training and $1000$ images per class for evaluation.

\textbf{Flickr} \cite{thomee16a}: The Yahoo/Flickr Creative Commons 100M dataset  consists of descriptive tags for around 100M images. We follow the experimental protocol in \cite{frogner15a} for the tag-prediction (multi-label) problem on $1\,000$ descriptive tags. The training and test sets consist of $10\,000$ randomly selected images associated with these tags. The features for images are extracted using MatConvNet \cite{vedaldi15a}. The train/test sets as well as the image features are available at \url{http://cbcl.mit.edu/wasserstein}.

\begin{table}[t]
  \caption{Performance of various OT distance based loss functions}\label{table:results}
  \begin{center}
    \begin{tabular}{lccccc}
      \toprule
      \multirow{2}{*}{Loss} & \multirow{2}{*}{$r$} & Animals & MNIST & \multicolumn{2}{c}{Flickr} \\ 
      & & (AUC) & (AUC) & (AUC) & (mAP) \\
      \midrule
      \multirow{1}{*}{${\rm W}_2^2$} \cite{frogner15a} & $-$ & $0.794$ & $0.848$ & $0.649$ & $0.0209$\\
      \addlinespace[0.3em]
      \multirow{1}{*}{${\rm SRW}_2^2$} \cite{paty19a} & $-$ & $0.802$ & $0.925$ & $0.734$ & $0.0370$\\
      \addlinespace[0.3em]
      \multirow{1}{*}{${\rm SRW}_5^2$} \cite{paty19a} & $-$ & $0.855$ & $0.860$ & $0.730$ & $0.0360$\\
      \addlinespace[0.3em]
      \multirow{1}{*}{${\rm SRW}_{10}^2$} \cite{paty19a} & $-$ & $0.860$ & $0.788$ & $0.721$ & $0.0350$\\
      \addlinespace[0.3em]
      \multirow{1}{*}{${\rm SRW}_{20}^2$} \cite{paty19a} & $-$ & $0.850$ & $0.788$ & $0.719$ & $0.0352$\\
      \addlinespace[0.3em]
      \multirow{1}{*}{${\rm RKP}_{4/3}$} \cite{dhouib20a} & $-$ & $0.854$ & $0.914$ & $0.732$ & $0.0366$\\
      \addlinespace[0.3em]
      \multirow{1}{*}{${\rm RKP}_{4}$} \cite{dhouib20a} & $-$ & $0.850$ & $0.860$ & $0.709$ & $0.0350$\\
      \addlinespace[0.3em]
      \multirow{1}{*}{${\rm RKP}_{\infty}$} \cite{dhouib20a} & $-$ & $0.837$ & $0.921$ & $0.735$ & $0.0370$\\
      \midrule
      \multirow{3}{*}{${\rm W_{P}}$, $k=1$  (Eq.~\ref{eqn:proposed-p-norm-min})} & $5$ & $0.801$ & $\mathbf{0.959}$ & $0.706$ & $0.0481$\\
      & $10$ & $0.907$ & $0.940$ & $0.745$  & $0.0625$\\
      & $20$ & $\mathbf{0.908}$ & $0.922$ & $\mathbf{0.768}$  & $\mathbf{0.0717}$\\
      \midrule
      \multirow{3}{*}{${\rm W_{P}}$, $k=2$ (Eq.~\ref{eqn:proposed-p-norm-min})} & $5$ & $0.890$ & $0.804$ & $0.760$  & $0.0625$\\
      & $10$ & $0.878$ & $0.702$ & $0.742$  & $0.0468$\\
      & $20$ & $0.874$ & $0.603$ & $0.737$  & $0.0441$\\
      \midrule
      \multirow{3}{*}{${\rm W_{KL}}$ (Eq.~\ref{eqn:proposed-min-kl-divergence})} & $5$ & $0.794$ & $0.867$ & $0.648$  & $0.0207$\\
      & $10$ & $0.749$ & $0.870$ & $0.649$  & $0.0205$\\
      & $20$ & $0.794$ & $0.779$ & $0.649$  & $0.0205$\\
      \midrule
      \multirow{3}{*}{${\rm W_{DS}}$ (Eq.~\ref{eqn:proposed-doubly-stochastic})} & $5$ & $\mathbf{0.908}$ & $0.927$ & $0.724$  & $0.0530$\\
      & $10$ & $0.878$ & $0.873$ & $0.668$  & $0.0325$\\
      & $20$ & $0.845$ & $0.813$ & $0.628$  & $0.0219$\\
      \bottomrule
    \end{tabular}
  \end{center}
\end{table}

\textbf{Experimental setup and baselines}: As described in Section~\ref{sec:ROTloss}, we use the fastText word embeddings \cite{bojanowski16a} corresponding to the labels for computing the OT ground metric in all our evaluations. 
We report the standard AUC metric for all the experiments. For the Flickr tag-prediction problem, we additionally report the mAP (mean average precision) metric. As the datasets are high dimensional, we use the low-dimensional modeling of the Mahalanobis metric $\bM$ (Section~\ref{sec:low_dimensional_metric_decomposition}) for the proposed robust OT distance-based loss functions by randomly grouping the features. We experiment with $r\in\{5,10,20\}$. 
We also report results with the following OT distance-based loss functions as baselines:
\begin{itemize}
    \item ${\rm W}_2^2$: the $2$-Wasserstein distance (\ref{eqn:W22distance}).
    \item ${{\rm SRW}}_k^2$: the subspace robust Wasserstein distance (\ref{eqn:SRW1}) proposed by Paty et al.~\cite{paty19a}. We obtain results for several hyper-parameter values $k=\{2,5,10,20\}$.
    \item ${\rm RKP}_\rho$: this class of robust Wasserstein distance \cite{dhouib20a} is parameterized by Schatten $\rho$-norm regularized Mahalanobis metric. We obtain results for several hyper-parameter values $\rho=\{4/3,4,\infty\}$. As shown in \cite{dhouib20a}, ${\rm RKP}_1={\rm W}_2^2$ and ${\rm RKP}_\infty={{\rm SRW}}_1^2$. 
\end{itemize}
Our experiments are run on a machine with 32 core Intel CPU ($2.1$ GHz Xeon) and a single NVIDIA GeForce RTX 2080 Ti GPU ($11$ GB). 
The model computations for all algorithms on all datasets are performed on the GPU, except on Flickr, where the baselines ${\rm RKP}_\rho$ and ${{\rm SRW}}_k^2$ have only CPU-based implementation. This is because the memory requirement for ${\rm RKP}_\rho$ and ${{\rm SRW}}_k^2$ on Flickr is too large for our GPU. 
The low-dimensional modeling for the proposed robust OT distances (with $r\ll d$) in Section~\ref{sec:low_dimensional_metric_decomposition} makes GPU implementation feasible on all datasets for ${\rm W_{P}}$, ${\rm W_{KL}}$, and ${\rm W_{DS}}$. Additional experimental details are in Appendix \ref{supp:sec:experiments}.

\begin{figure}[t]
\center
\includegraphics[width=.6\textwidth]{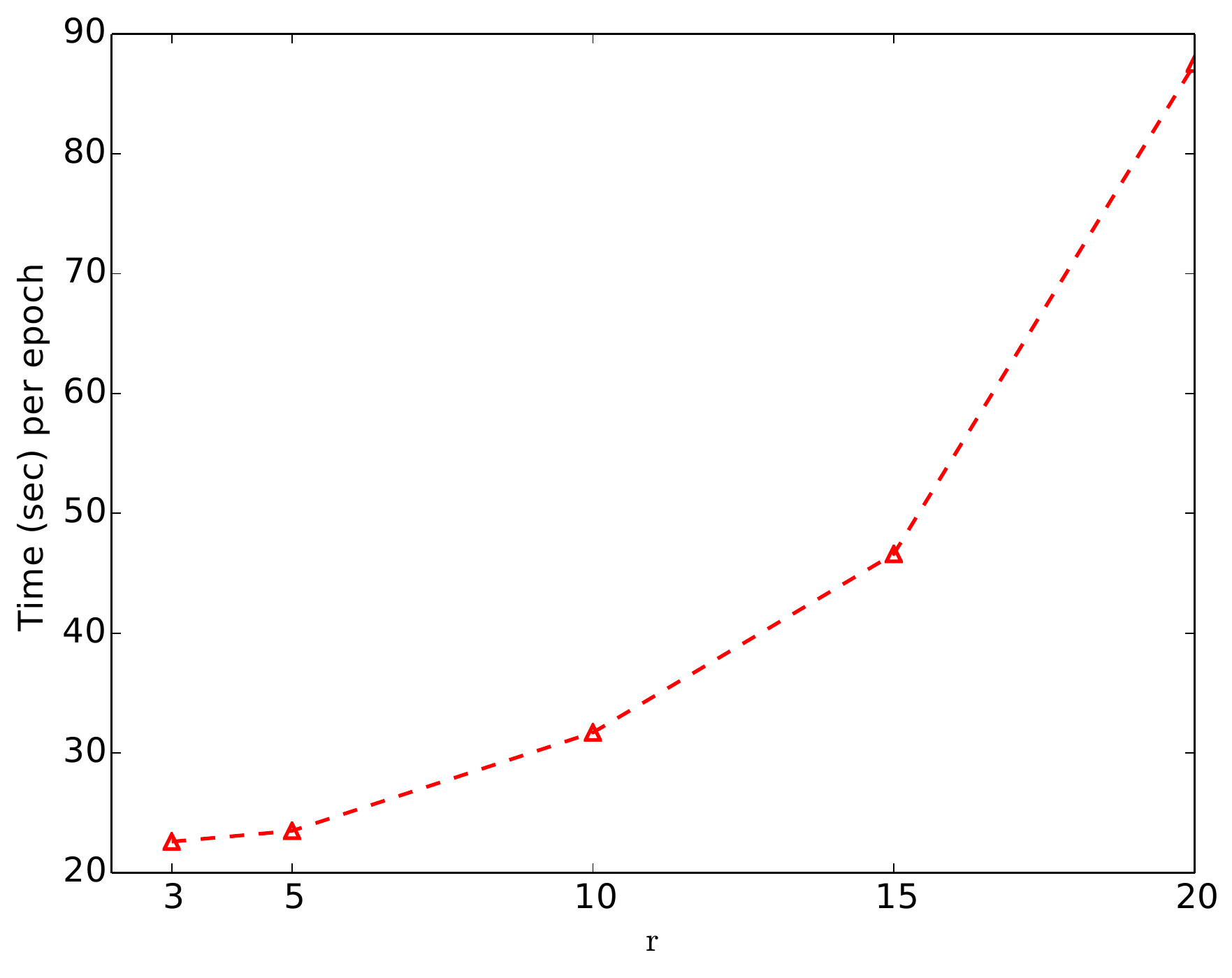}
\caption{Time taken per epoch versus $r$ on the Flickr dataset by our algorithm with the proposed ${\rm W}_{\rm P}$ loss with $k=1$.}\label{fig:timing}
\end{figure}

\subsection{Results and discussion}

\textbf{Generalization performance}: We report the results of our experiments in Table~\ref{table:results}. Overall, we observe that the proposed robust OT distance based loss functions provide better generalization performance than the considered baselines. We also evaluate the robustness of proposed OT distance-based loss functions with respect to randomized grouping of the features into $r$ groups. As discussed in Section~\ref{sec:low_dimensional_metric_decomposition}, the low-dimensional modeling of $\bM$ requires dividing the $d$ features into $r$ groups. Results on the proposed robust OT distance in Table~\ref{table:results} are obtained by one randomized grouping of features. In Table~\ref{appendix:table:robust_r_results} in Appendix~\ref{supp:sec:experiments}, we report the mean AUC and the corresponding standard deviation obtained across five random groupings of features on the same train-test split (for the smaller Animals and MNIST datasets). We observe that the results across various randomized feature groupings is quite stable, signifying the robustness of the proposed OT distance-based loss functions with regards to groupings of features. In Section \ref{supp:movies_dataset}, we additionally show comparative results of the SRW and ${\rm W_P}$ distance on a movies dataset.


\textbf{Computation timing}: We also report the time taken per epoch of the SGD algorithm on the Flickr dataset with ${\rm W}_{\rm P}$ $(k=1)$ as the loss function. From Figure~\ref{fig:timing}, we observe that our model scales gracefully as $O(r^2)$ as discussed in Section~\ref{sec:algorithms}. 
We next compare per-epoch time take by our ${\rm W_{P}}$ loss with $r=20$ against those taken by the robust OT baselines ${{\rm SRW}}_k^2$ and ${\rm RKP}_\rho$. The per-epoch time taken by ${{\rm SRW}}_k^2$ and ${\rm RKP}_\rho$ is $14$-$16$s for Animals, $23$-$26$s for MNIST, and around $12.5$ hours for Flickr. The per-epoch time taken by our W$_{\rm P}$ loss with $r=20$ is $0.3$s (Animals), $0.5$s (MNIST), and $87.5$s (Flickr). Hence, on Animals and MNIST datasets, computing the proposed robust OT distances is at least $45$ times faster than both ${{\rm SRW}}_k^2$ and ${\rm RKP}_\rho$. 
This is due to a) efficient computation of optimal $\bM$ (Section~\ref{sec:formulations}) and b) low-dimensional modeling of $\mathbf{M}$ (Section~\ref{sec:low_dimensional_metric_decomposition}).

\section{Conclusion}
We discussed robust optimal transport problems arising from Mahalanobis parameterized cost functions. A particular focus was on discussing novel formulations that may be computed efficiently. We also proposed a low-dimensional modeling of the Mahalanobis metric for efficient computation of robust OT distances. 
An immediate outcome is that it allows the use the robust OT formulations in high-dimensional multi-class/multi-label settings. Results on real-world datasets demonstrate the efficacy of the proposed robust OT distances in learning problems. 



\appendix

\section{Proofs}\label{supp:sec:proofs}

\subsection{Proof of Theorem~\ref{thm:p-norm}} \label{supp:sec:proof_pnorm}

Consider (\ref{eqn:fgamma}) with $\calM=\{\bM:\bM\succeq\bzero, \norm{\bM}_{p}\leq1\}$, i.e.,
\begin{equation}\label{appendix:eqn:fgamma-p-norm}
    f(\gamma) = \maxop_{\bM\succeq\bzero,\norm{\bM}_{p}\leq 1}\ \inner{\bV_\gamma,\bM}.
\end{equation}
Now, consider the relaxed problem of (\ref{appendix:eqn:fgamma-p-norm}) {\it without} the PSD constraint:
\begin{equation}\label{appendix:eqn:fgamma-relaxed-p-norm}
    \bM^{*}(\gamma) = \argmax_{\norm{\bM}_{p}\leq 1}\ \inner{\bV_\gamma,\bM}.
\end{equation}
From the H{\"{o}}lder's inequality for vectors, the optimal solution of (\ref{appendix:eqn:fgamma-relaxed-p-norm}) is $\bM^{*}(\gamma)=\norm{\bV_\gamma}_{q}^{-q/p}(\bV_\gamma)^{\circ(q/p)},$ where $q$-norm is the dual of $p$-norm, i.e., $1/p+1/q=1$. 

For $k\in\mathbb{N}$ and  $p=\frac{2k}{2k-1}$, we have $q=\frac{1}{2k}$. Therefore, $\bM^{*}(\gamma)=\norm{\bV_\gamma}_{2k}^{1-2k}(\bV_\gamma)^{\circ(2k-1)}$. It should be noted that the $\bM^{*}(\gamma)$ for such $p$-norm is a symmetric PSD matrix (via the Schur product theorem) as $\bV_\gamma$ is a symmetric PSD matrix. Consequently,  $\bM^{*}(\gamma)$ is also the optimal solution of (\ref{appendix:eqn:fgamma-p-norm}). Substituting this $\bM^{*}(\gamma)$ in (\ref{appendix:eqn:fgamma-p-norm}) leads to (\ref{eqn:proposed-p-norm-min}), thereby proving Theorem \ref{thm:p-norm}. A similar result was proved in \cite[Theorem 4]{mjaw15a} in the context of multitask learning.

\subsection{Proof of Lemma~\ref{lemma:feature-selection}}\label{supp:sec:proof_simplex}
Consider the inner maximization problem in (\ref{eqn:feature-selection}): $\maxop_{\bm\in\Delta_d} \inner{\gamma,\bC_{\bm}}-\lambda_{\bm} {{\rm D}_{\rm KL}}(\bm,\bone)$. This involves the entropy regularization as well as the simplex constraint on $\bm$. Such problems are well studied in literature \cite{bental01a} and the optimal solution of the above can be obtained via the Lagrangian duality. The $i$-th coordinate of its optimal solution $\bm^{*}$ is $\bm^{*}(i)=e^{\bv_i/\lambda_\bm}/(\sum_{j=1}^d e^{\bv_j/\lambda_\bm})$, where $\bv_i$ is the $i$-th diagonal entry of the matrix $\bV_\gamma$. Plugging $\bm^{*}$ in the objective function of (\ref{eqn:feature-selection}) allows to equivalently write (\ref{eqn:feature-selection}) as 
\begin{equation}\label{appendix:eqn:feature-selection}
\minop_{\gamma\in\Pi(\bp,\bq)}\ \lambda_\bm (\ln{(\sum_{i} e^{\bv_i/\lambda_\bm})} - (d-1)).
\end{equation}
Now consider the objective in (\ref{eqn:proposed-min-kl-divergence}) with $\bM_0=\bI$ and compare it with the objective in (\ref{appendix:eqn:feature-selection}) after leaving out the constants. We observe that the objective in (\ref{appendix:eqn:feature-selection}) is the  natural logarithm of the objective in (\ref{eqn:proposed-min-kl-divergence}). Since logarithm is an increasing function, the argmin solutions of (\ref{eqn:proposed-min-kl-divergence}) and (\ref{appendix:eqn:feature-selection}) are be the same.

\subsection{Proof of Theorem~\ref{thm:doubly-stochastic-regularization}} \label{supp:sec:proof_DS}
Consider a relaxed variant of (\ref{eqn:proposed-minimax-doubly-stochastic}) without the PSD constraint:
\begin{equation}\label{appendix:eqn:lightspeed}
\begin{array}{l}
\bM^{*}(\gamma) 
= 
\argmax\limits_{\bM\geq\bzero,\bM\bone=\bone,\bM^\top\bone=\bone} \inner{\bV_\gamma,\bM} \\
\qquad \qquad \qquad \qquad \qquad \qquad  - \lambda_{\bM}{{\rm D}_{\rm KL}}(\bM,\bM_0).
\end{array}
\end{equation}

Problem (\ref{appendix:eqn:lightspeed}) has a form similar to the entropic-regularized optimal transport problem studied in \cite{cuturi13a}, but with a symmetric cost matrix as $-(\bV_{\gamma}+\ln(\bM_0))$. As discussed in \cite{cuturi13a}, the solution of the the entropic-regularized optimal transport problem can be efficiently obtained using the Sinkhorn-Knopp algorithm \cite{knight08a}. Due to the symmetric cost matrix, the optimal solution of (\ref{appendix:eqn:lightspeed}) via the Sinkhorn-Knopp algorithm has the following form \cite{knight08a}: 
\begin{equation}\label{supp:eqn:M_form_DS}
\bM^{*}(\gamma)=\bD\left( \bM_0 \odot e^{\circ(\bV_\gamma/\lambda_{\bM})}\right)\bD,
\end{equation}
where $\bD$ is a diagonal matrix with positive entries \cite[Lemma 2 proof]{cuturi13a} and \cite[Section 3]{knight08a}. 

It should be noted that the optimal solution $\bM^{*}(\gamma)$ of (\ref{appendix:eqn:lightspeed}), whose expression is given in (\ref{supp:eqn:M_form_DS}), is also a symmetric PSD matrix as $\bM_0 \odot e^{\circ(\bV_\gamma/\lambda_{\bM})}$ is a symmetric PSD matrix (refer to the proof of Theorem~\ref{thm:kl-regularization}). Hence, $\bM^{*}(\gamma)$ is also the solution of (\ref{eqn:proposed-minimax-doubly-stochastic}).



\subsection{Proof of Lemma \ref{lemma:gradient_rot}} \label{supp:sec:proof_gradient_rot}

Denoting $\mu_1 = h(\bx_j;\bW)$ and  $\mu_2 = \hat{\by}_j$, we have 
\begin{equation}\label{supp:eqn:rotdual}
\begin{array}{lll}
  {\rm W}_{\rm ROT}(h(\bx_j;\bW),\hat{\by}_j) \\
  =  \minop_{\gamma \in \Pi(\mu_1, \mu_2)} f(\gamma) +  \lambda_{\gamma} \langle \gamma, {\rm ln}(\gamma) \rangle \\
   =  \maxop_{
   \nu_1, \nu_2 \in \mathbb{R}^L,
   \Delta \in \mathbb{R}^{L \times L} \geq 0
   }  \ \minop_{\gamma \in  \mathbb{R}^{L \times L}}  f(\gamma) \\ + \lambda_{\gamma} \langle \gamma, {\rm ln}(\gamma) \rangle 
   \\ - \langle \nu_1, \gamma \bone - \mu_1\rangle - \langle \nu_2, \gamma^{\top} \bone - \mu_2\rangle - \langle \Delta, \gamma \rangle,
  \end{array}
\end{equation}
where $\bone$ is the column vector of ones of size $L$, $\langle \cdot, \cdot \rangle$ is the standard inner product between matrices, $\nu_1$, $\nu_2$, and $\Delta$ are the dual variables, and $f$ is the non-linear convex function obtained as $f(\gamma) = \max_{\bM \in \mathcal{M}} \langle \bV_{\gamma}, \bM \rangle$. The last equality in (\ref{supp:eqn:rotdual}) comes from strong duality.

Our interest is to compute the gradient $\nabla_{\mu_1} {\rm W}_{\rm ROT}(\mu_1,\mu_2)$ of ${\rm W}_{\rm ROT}(\mu_1,\mu_2)$ with respect to $\mu_1$. Given the optimal solution $(\nu_1^*, \nu_2^*, \Delta^*, \gamma^*, \bM^*(\gamma^*))$, the gradient has the expression
\begin{equation}\label{supp:eqn:dual_gradient}
\nabla_{\mu_1} {\rm W}_{\rm ROT}(\mu_1,\mu_2) = \nu_1^* - \frac{\langle \nu_1^*, \bone \rangle}{L}\bone,
\end{equation}
where $\nu_1^*$ is the partial derivative of (\ref{supp:eqn:rotdual}) with respect to $\mu_1$ and the second term $\frac{\langle \nu_1^*, \bone \rangle}{L}\bone$ is the normal component of $\nu_1$ to the simplex set $\bone ^\top \mu_1  = 1$. Overall, $\nabla_{\mu_1} {\rm W}_{\rm ROT}(\mu_1,\mu_2)$ is {\it tangential} to the simplex $\bone ^\top \mu_1  = 1$ at $\mu_1$.

To compute the expression for the right hand side of ({\ref{supp:eqn:dual_gradient}}), we look at the optimality conditions of (\ref{supp:eqn:rotdual}), i.e.,
\begin{equation}\label{supp:eqn:kkt}
\begin{array}{lll}
\nabla_{\gamma} f (\gamma^*) + \lambda_{\gamma}({\rm ln}(\gamma^*) + \bone \bone^\top) \\
\qquad \qquad \qquad - \nu_1^* \bone ^\top - \bone {\nu_2^*} ^\top - \Delta^* = \bzero.
\end{array}
\end{equation}
Here, $\Delta ^* = \bzero$ as $\gamma^* > 0$ (entropy regularization and complementary slackness). Consequently, (\ref{supp:eqn:kkt}) boils down to
\begin{equation}\label{supp:eqn:kkt2}
\begin{array}{lll}
\Rightarrow  \nu_1^* \bone ^\top + \bone {\nu_2^*} ^\top = \bA  ,
\end{array}
\end{equation}
where $\bA = \nabla_{\gamma} f (\gamma^*) + \lambda_{\gamma}({\rm ln}(\gamma^*) + \bone \bone^\top) = \bC^* + \lambda_{\gamma}({\rm ln}(\gamma^*) + \bone \bone^\top)$. Here, $\bC^*(\bl_p,\bl_q) = (\bl_p - \bl_q)^\top \bM^*(\gamma^*) (\bl_p - \bl_q)$ and $\bl_p$ is the ground embedding for $p$-th label. From (\ref{supp:eqn:kkt2}), $\nu_1^*$ and $\nu_2^*$ are translation invariant, i.e., $\nu_1^* + \alpha \bone$ and $\nu_2^* - \alpha \bone$ are solutions for all $\alpha \in \mathbb{R}$. However, we are interested not in $\nu_1^*$, but in $\nu_1^* - \frac{\langle \nu_1^*, \bone \rangle}{L}\bone$, which is unique (as its mean is $\bzero$). 

The term $\nu_1^* - \frac{\langle \nu_1^*, \bone \rangle}{L}\bone$ can be computed directly by eliminating $\nu_2^*$ in (\ref{supp:eqn:kkt2}) using basic operations (pre- and post-multiplication with $\bone$) to obtain
\begin{equation*}
    \nu_1^* - \frac{\langle \nu_1^*, \bone \rangle}{L}\bone = \frac{1}{L}\bA \bone - \frac{\bone ^\top \bA \bone}{L^2} \bone.
\end{equation*}
This completes the proof of the lemma.

\begin{table}[ht]
  \caption{Average performance on five random groupings of features}\label{appendix:table:robust_r_results}
  \begin{center}
    \begin{tabular}{lccccc}
      \toprule
      \multirow{2}{*}{Loss} & \multirow{2}{*}{$r$} & Animals & MNIST \\ 
      & & (AUC) & (AUC) \\
      \midrule
      \multirow{3}{*}{${\rm W_{P}}$, $k=1$  (Eq.~\ref{eqn:proposed-p-norm-min})} & $5$ & $0.805 \pm 0.003$ & $\mathbf{0.957} \pm 0.001$\\
      & $10$ & $0.909 \pm 0.002$ & $0.939 \pm 0.001$\\
      & $20$ & $\mathbf{0.911} \pm 0.002$ & $0.920 \pm 0.002$\\
      \midrule
      \multirow{3}{*}{${\rm W_{P}}$, $k=2$ (Eq.~\ref{eqn:proposed-p-norm-min})} & $5$ & $0.892 \pm 0.001$ & $0.787 \pm 0.010$ \\
      & $10$ & $0.880 \pm 0.001$ & $0.700 \pm 0.002$ \\
      & $20$ & $0.875 \pm 0.002$ & $0.601 \pm 0.001$ \\
      \midrule
      \multirow{3}{*}{${\rm W_{KL}}$ (Eq.~\ref{eqn:proposed-min-kl-divergence})} & $5$ & $0.799 \pm 0.003$ & $0.806 \pm 0.034$ \\
      & $10$ & $0.799 \pm 0.003$ & $0.801 \pm 0.039$ \\
      & $20$ & $0.799 \pm 0.003$ & $0.781 \pm 0.001$ \\
      \midrule
      \multirow{3}{*}{${\rm W_{DS}}$ (Eq.~\ref{eqn:proposed-doubly-stochastic})} & $5$ & $\mathbf{0.911} \pm 0.002$ & $0.928 \pm 0.002$ \\
      & $10$ & $0.884 \pm 0.003$ & $0.868 \pm 0.010$ \\
      & $20$ & $0.855 \pm 0.006$ & $0.817 \pm 0.013$ \\
      \bottomrule
    \end{tabular}
  \end{center}
\end{table}





\subsection{Additional details on experiments and results}\label{supp:sec:experiments}
All the multi-class/label learning experiments are performed in a standard setting, where the fastText embeddings are unit normalized (via $2$-norm), the Sinkhorn algorithm is run for $10$ iterations, the FW algorithm is run for $1$ iteration (our initial experiments showed that a single FW iteration resulted in a good quality convergence), and $\lambda_{\gamma}$ is $0.02$ in (\ref{eqn:rotminmax}), and $\lambda_{\beta}$ (LMO step in Algorithm~\ref{alg:FW}) is $0.2$. 
Following \cite{frogner15a}, we regularize the softmax model parameters by $0.0005 \|{\rm \bW} \|_2 ^2$ in Problem (\ref{eqn:rotminmax}).

Table~\ref{appendix:table:robust_r_results} shows additional results with randomized feature groups.

\subsection{Movies dataset}\label{supp:movies_dataset}

We also comparatively study the subspace the robust Wasserstein (SRW) distance \cite{paty19a} and the ${\rm W_P}$ distance (with $r=10,k=1$) between the scripts of seven movies. We follow the experimental protocol described in \cite{paty19a}. The marginals are the histograms computed from the word frequencies in the movie scripts and each word is represented as a $300$-dimensional fastText embedding \cite{bojanowski16a}. 

It should be noted that the range/spread of SRW (\ref{eqn:SRW1}) and ${\rm W_P}$ (\ref{eqn:proposed-p-norm-min}) distances are different for the same movie. Hence, 
Table \ref{appendix:table:movie_results} reports the normalized SRW and the normalized ${\rm W_P}$ distances for all pairs of movies. The normalization is done column-wise as follows: we divide all the SRW distances in a column by the maximum SRW distance in that column (and similarly normalize the ${\rm W_P}$ distances as well). This normalization ensures that for a given movie (representing the column), the minimum relative distance is $0$ (with itself) while the maximum relative distance is $1$ for both SRW and ${\rm W_P}$. 

We observe that both SRW and ${\rm W_P}$ are usually consistent in selecting the closest movie (i.e., the row corresponding with minimum non-zero distance in a column). However, ${\rm W_P}$ tends to have a wider spread of distances, i.e., the difference in the distances corresponding to the closest and the furthest movies. As an example, SRW computes similar distances for the pairs (Kill Bill Vol.1, Kill Bill Vol.2) and (Inception, The Martian) while ${\rm W_P}$ gives the (Kill Bill Vol.1, Kill Bill Vol.2) pair a much lower relative distance (which seems more reasonable as they are sequels).


\begin{table*}[ht]
  \caption{Normalized distances between movie scripts. Each cell indicates the ${\rm SRW}_2^2$/${\rm W_P}$ distances respectively. Here, D = Dunkirk, G = Gravity, I = Interstellar, KB1 = Kill Bill Vol.1, KB2 = Kill Bill Vol.2, TM = The Martian, and T = Titanic.}\label{appendix:table:movie_results}
  \begin{center}
\scriptsize
    \begin{tabular}{|c|c|c|c|c|c|c|c|}
      \hline
 &
  D &
  G &
  I &
  KB1 &
  KB2 &
  TM &
  T \\
  \hline
D &
  0.000/0.000 &
  0.906/0.943 &
  0.911/0.951 &
  0.995/1.000 &
  0.995/0.998 &
  0.964/1.000 &
  $\mathbf{0.924}$/0.931 \\
  \hline
G &
  0.911/0.931 &
  0.000/0.000 &
  0.847/0.880 &
  1.000/1.000 &
  1.000/1.000 &
  0.907/0.858 &
  1.000/0.978 \\
  \hline
I &
  0.916/0.901 &
  $\mathbf{0.847}$/$\mathbf{0.846}$ &
  0.000/0.000 &
  0.995/0.953 &
  1.000/0.961 &
  $\mathbf{0.876}$/$\mathbf{0.814}$ &
  0.978/$\mathbf{0.898}$ \\
  \hline
KB1 &
  0.965/0.972 &
  0.966/0.985 &
  0.961/0.978 &
  0.000/0.000 &
  $\mathbf{0.808}$/$\mathbf{0.673}$ &
  0.984/0.964 &
  0.973/0.945 \\
  \hline
KB2 &
  1.000/0.984 &
  1.000/1.000 &
  1.000/1.000 &
  $\mathbf{0.837}$/$\mathbf{0.683}$ &
  0.000/0.000 &
  1.000/0.960 &
  0.978/0.948 \\
  \hline
TM &
  0.921/1.000 &
  0.862/0.870 &
  $\mathbf{0.833}$/0.859 &
  0.969/0.992 &
  0.951/0.973 &
  0.000/0.000 &
  0.989/1.000 \\
  \hline
T &
  $\mathbf{0.842}$/$\mathbf{0.814}$ &
  0.906/0.867 &
  0.887/$\mathbf{0.828}$ &
  0.913/0.849 &
  0.887/0.840 &
  0.943/0.874 &
  0.000/0.000 \\
  \hline

    \end{tabular}
  \end{center}
\end{table*}

\bibliographystyle{unsrt}
\bibliography{minimaxOT}

\end{document}